\documentclass[11pt]{article}
\usepackage{amsmath, amssymb, amsthm, booktabs, hyperref, threeparttable}
\usepackage[margin=1in]{geometry}
\usepackage{fullpage}
\usepackage[ruled,vlined]{algorithm2e}
\SetAlCapSkip{1em}
\usepackage{hyperref}
\usepackage{cite}
\usepackage{graphicx}
\usepackage{booktabs}
\usepackage{tabularx}

\newtheorem{theorem}{Theorem}[section]
\newtheorem{lemma}{Lemma}[section]
\newtheorem{remark}{Remark}[section]

\title{Robust Randomized Low-Rank Approximation \\ with Row-Wise Outlier Detection}
\author{Aidan Tiruvan}
\date{\today}

\begin{document}
\maketitle
\begin{abstract}

\noindent We present a method for robust randomized low-rank approximation in the presence of \emph{row-wise} adversarial corruption, a setting where an \(\alpha\)-fraction of the rows in a data matrix can deviate substantially from the majority. Concretely, we consider
\[
A = B + N \;\in\; \mathbb{R}^{m \times n},
\]
where \(B\) is (approximately) rank-\(k\) and represents the ``clean'' portion, and \(N\) is a noise matrix that remains bounded on most rows but may be large on an \(\alpha\)-fraction of rows. Traditional robust PCA methods (e.g., Principal Component Pursuit for entrywise outliers or Outlier Pursuit for entire corrupted columns) often rely on convex optimization or multiple iterative steps, which can be computationally heavy for large-scale data. By contrast, we propose a \emph{single-pass} randomized procedure that scales linearly in the number of rows and is straightforward to implement.
\par\vspace{0.5em}
\noindent In our theoretical framework, every clean row satisfies 
\[
\|B_{i,:}\|_2 \;\ge\; \kappa\,\delta \quad (\kappa>1),
\]
while adversarial rows are separated from the clean ones by a gap parameter \(\Delta\). Equivalently, defining 
\[
\gamma \;=\; \frac{\Delta}{\max_{i\in S_{\mathrm{clean}}}\|B_{i,:}\|_2},
\]
we assume \(\gamma\) is large, meaning outlier rows exhibit substantially larger norm than any inlier row. We first reduce dimensionality via a Johnson--Lindenstrauss (JL) projection, which preserves the norms of non-outlier rows up to \((1\pm \varepsilon)\). This dimension reduction is crucial for handling high-dimensional data efficiently. We then employ robust statistics---specifically median and MAD---to estimate the typical row norm distribution and detect outliers with a simple threshold. 

\par\vspace{0.75em}
\noindent \textbf{Row-Wise Concentration Analysis:} We provide refined second-order bounds (Lemma~\ref{lem:row_concentration}) showing that all \((1-\alpha)m\) inlier rows maintain their norms within the desired JL distortion. 
\par\vspace{0.75em}
\noindent \textbf{Robust Threshold Guarantee:} Under a sufficiently large norm gap \(\gamma\), our thresholding step reliably separates adversarial rows with high probability (Lemma~\ref{lem:outlier_detection}). 
\par\vspace{0.75em}
\noindent \textbf{Near-Optimal Low-Rank Approximation:} After removing the corrupted rows, we compute a rank-\(k\) approximation on the retained data. The resulting approximation to the original clean matrix \(B\) is near-optimal up to a small additive term, ensuring that low-rank structure is ultimately preserved.
\par\vspace{0.75em}
\noindent In essence, our approach combines the benefits of \emph{random projection-based sketching} with \emph{robust statistics} to yield a one-pass, computationally efficient solution to robust PCA in the row-wise outlier regime.

\end{abstract}

\section{Introduction}
Low-rank approximation underpins a variety of high-dimensional data analysis tasks, including principal component analysis (PCA), matrix completion, and subspace estimation. When a fraction of rows in the data matrix is adversarially corrupted, classical PCA methods may fail dramatically: even a small number of outliers can skew the low-rank components. Numerous approaches to \emph{robust} PCA have been proposed, often targeting either entrywise or columnwise outliers via convex optimization or other computationally demanding frameworks.
\par\vspace{0.75em}
\noindent A pivotal example is Principal Component Pursuit (PCP) introduced by Candès et al.\ (2011)~\cite{candes2011robust}, which handles \emph{entrywise} sparse outliers through a matrix decomposition \(A = L + S\) with nuclear and \(\ell_1\)-norm penalties. In parallel, variants of \emph{Outlier Pursuit} \cite{chen2013robust} tackle entire corrupted columns or rows by incorporating an \(\ell_{2,1}\)-norm penalty on the outlier block, guaranteeing subspace recovery under certain conditions. While such formulations can yield strong recovery guarantees, they often rely on iterative solvers with high complexity, posing scalability challenges for very large data. Other non-convex methods, including robust M-estimators for PCA \cite{zhang2014novel}, likewise feature solid theoretical underpinnings but may involve repeated SVDs or non-convex optimizations, limiting feasibility in massive-scale scenarios.
\par\vspace{0.75em}
\noindent A complementary direction is \emph{Coherence Pursuit}, proposed by Rahmani and Atia (ICML 2017), which detects outliers by measuring the coherence (inner-product similarities) among data points. Although Coherence Pursuit is non-iterative and can tolerate a large fraction of outliers, it requires pairwise comparisons that scale poorly (\(O(m^2 n)\)) when the number of samples \(m\) is extremely large. Meanwhile, a range of classical robust statistics methods (e.g., L1-PCA, trimmed approaches, or MCD estimators) can detect outliers but are often limited by breakdown points below 50\% or by iterative updates that become expensive as the dataset grows.
\par\vspace{0.75em}
\noindent By contrast, our work focuses on the \emph{row-wise} outlier setting, where an \(\alpha\)-fraction of rows can be heavily corrupted, with \(\alpha < 0.5\). We target scenarios where these outliers have substantially larger \(\ell_2\)-norm than any inlier row, making them amenable to \emph{simple thresholding} once we have a reliable estimate of the inlier distribution. This approach aligns with practical cases in which corrupted rows exhibit abnormally large magnitude (e.g., sensor or hardware malfunctions), and it bypasses the more general but costlier machinery of convex relaxations or pairwise coherence checks.

\par\vspace{0.75em}

\noindent We model the data as:
\[
A = B + N,
\]
where \(B\in\mathbb{R}^{m\times n}\) is (approximately) rank-\(k\), representing the clean portion, and \(N\) is a noise matrix satisfying \(\|N_{i,:}\|_2 \le \delta\) for all \emph{clean} rows \(i\in S_{\mathrm{clean}}\). We assume that at most an \(\alpha\)-fraction of rows are adversarially corrupted, denoted by \(S_{\mathrm{adv}}\). Formally, for each \(i\in S_{\mathrm{clean}}\),
\[
\|B_{i,:}\|_2 \;\ge\; \kappa\,\delta,
\]
with \(\kappa>1\), ensuring each clean row has sufficiently large signal relative to noise. We also posit that each adversarial row \(j\) satisfies
\[
\|A_{j,:}\|_2 \;\ge\; \max_{i \in S_{\mathrm{clean}}}\|B_{i,:}\|_2 + \Delta,
\]
or equivalently a large normalized gap
\[
\gamma \;=\; \frac{\Delta}{\max_{i\in S_{\mathrm{clean}}}\|B_{i,:}\|_2}.
\]
Such a gap indicates that outlier rows lie well outside the norm range of inliers, facilitating near-perfect separation with high probability.

\par\vspace{0.75em}

\noindent We apply a \emph{Johnson--Lindenstrauss} (JL) random projection \(\Psi\in\mathbb{R}^{n\times s}\) to reduce the column dimension from \(n\) to \(s \approx O\bigl(\tfrac{1}{\varepsilon^2}\log(\tfrac{(1-\alpha)m}{\delta'})\bigr)\). Under suitable conditions, the norms of all clean rows are preserved within \((1\pm \varepsilon)\), while the adversarial rows, being large in norm, remain distinguishable from inliers. We then use robust statistics (specifically, the median and MAD \cite{rousseeuw1993alternatives}) to estimate the typical row-norm distribution among the clean data. A simple threshold \(\tau = \widehat{\mu} + c\,\widehat{\sigma}\) (for some constant \(c\)) flags rows exceeding \(\tau\) as outliers. Crucially, since \(\alpha<0.5\), the median-based estimator remains stable in the face of adversarial contamination.

\par\vspace{0.75em}

\noindent This paper’s theoretical results (Lemmas~\ref{lem:row_concentration}--\ref{lem:outlier_detection}) give quantitative bounds on the probability of retaining all clean rows while discarding all adversaries, supplemented by an error analysis showing that once corrupted rows are removed, we can recover a near-optimal rank-\(k\) approximation to \(B\). In contrast to many prior robust methods, our scheme is essentially \emph{non-iterative} and can be parallelized for large-scale or streaming applications, thanks to the linear cost of multiplying by the sketch \(\Psi\).

\subsection*{Notation}
For a vector \(x\), \(\|x\|_2\) denotes its Euclidean norm. For a matrix \(M\), \(M_{i,:}\) is the \(i\)-th row. Let \(m\) be the number of rows, \(n\) the number of columns, and \(\alpha\) the maximum fraction of adversarial rows. We denote the set of clean rows by \(S_{\mathrm{clean}} \subseteq [m]\) and the set of adversarial rows by \(S_{\mathrm{adv}} = [m]\setminus S_{\mathrm{clean}}\). Throughout, \(\delta\) represents the noise bound for inliers, and \(\Delta\) (or \(\gamma\)) captures the outlier norm gap.

\section{Preliminaries}
We consider a data matrix \(A \in \mathbb{R}^{m \times n}\) decomposed as
\[
A = B + N,
\]
where \(B \in \mathbb{R}^{m \times n}\) is an (approximately) rank-\(k\) ``clean'' component, and \(N\) is a noise matrix. For each row \(i \in S_{\mathrm{clean}}\) (the set of non-adversarial rows), we have 
\[
\|N_{i,:}\|_2 \;\le\; \delta,
\]
so that the noise on any \emph{clean} row remains bounded. We assume at most an \(\alpha\)-fraction of rows are adversarial (\(\alpha < 0.5\) to ensure median-based estimators remain stable). Concretely, if \(\vert S_{\mathrm{adv}}\vert = \alpha m\) denotes the adversarial set, then for all \(i \in S_{\mathrm{clean}}\),
\[
\|B_{i,:}\|_2 \;\ge\; \kappa \,\delta \quad (\kappa>1),
\]
ensuring each clean row has nontrivial ``signal'' relative to the noise level~\(\delta\). Meanwhile, every adversarial row \(j \in S_{\mathrm{adv}}\) is assumed to satisfy
\[
\|A_{j,:}\|_2 \;\ge\; \max_{i\in S_{\mathrm{clean}}}\|B_{i,:}\|_2 \;+\; \Delta,
\]
implying a large \(\ell_2\)-norm gap:
\[
\gamma \;=\; \frac{\Delta}{\max_{i \in S_{\mathrm{clean}}}\|B_{i,:}\|_2}.
\]
As discussed in the introduction, this condition pinpoints scenarios where outliers have substantially higher norms than any of the clean rows.

\paragraph{Random Projection \(\Psi\)}
To handle large \(n\) efficiently, we adopt a Johnson--Lindenstrauss (JL) random projection \(\Psi \in \mathbb{R}^{n \times s}\). Each clean row \(A_{i,:} = B_{i,:} + N_{i,:}\) then maps to
\[
S_{i,:} \;=\; A_{i,:}\,\Psi \;\in\; \mathbb{R}^{s}.
\]
A broad family of JL distributions can be used:
\begin{itemize}
    \item \emph{i.i.d.~Gaussian}:\;\(\Psi_{ij}\sim \mathcal{N}(0,1/\sqrt{s})\).
    \item \emph{Sparse Rademacher}:\;\(\Psi_{ij} \in \{0,\pm \tfrac{1}{\sqrt{s}}\}\) with a small fraction of nonzeros~\cite{achlioptas2003database}.
    \item \emph{Structured transforms} such as the Fast Johnson--Lindenstrauss Transform~\cite{vempala2005random}.
\end{itemize}
In all cases, one can select
\[
s \;=\; O\!\Bigl(\frac{1}{\varepsilon^2}\;\log\!\bigl(\tfrac{(1-\alpha)m}{\delta'}\bigr)\Bigr)
\]
to guarantee, with probability at least \(1 - \delta'\), that \emph{all} clean rows are norm-preserved within \((1\pm\varepsilon)\).  

\paragraph{Implementation Details}

\par\vspace{0.75em}

\noindent \\ \\ \textit{Efficient Sketching:} Sparse Rademacher transforms reduce computational cost, especially when \(A\) is large or sparse. Multiplying \(A\) by \(\Psi\) can be done in \(\widetilde{O}(\mathrm{nnz}(A))\) time.

\par\vspace{0.75em}

\noindent \emph{Robust Estimators at Scale:} We employ median and MAD to detect outliers; for large \(m\), fast approximation schemes are available for these estimators~\cite{rousseeuw1993alternatives}.
    
\par\vspace{0.75em}

\noindent Once the projection is computed, our method thresholds the projected row norms to distinguish adversarial from clean rows, leveraging the large gap \(\gamma\). The subsequent sections detail how we derive high-probability separation guarantees using classical concentration tools and refined bounds for random sketches.

\section{Row-Wise Concentration of the Sketch}
\label{sec:row_concentration}

\begin{lemma}[Row-Wise Concentration]\label{lem:row_concentration}
Let \(B \in \mathbb{R}^{m \times n}\) be an approximately rank-\(k\) matrix, and let \(N \in \mathbb{R}^{m \times n}\) satisfy \(\|N_{i,:}\|_2 \le \delta\) for all \(i \in S_{\mathrm{clean}}\). Define \(A = B + N\). Suppose \(\|B_{i,:}\|_2 \ge \kappa\,\delta\) for some \(\kappa > 1\). Let \(\Psi \in \mathbb{R}^{n \times s}\) be a Johnson--Lindenstrauss (JL) projection matrix with
\[
s \;=\; O\!\Bigl(\tfrac{1}{\varepsilon^2}\,\log\!\bigl(\tfrac{(1 - \alpha)m}{\delta'}\bigr)\Bigr).
\]
Then, with probability at least \(1 - \delta'\), every clean row \(i \in S_{\mathrm{clean}}\) satisfies
\[
(1 - \varepsilon)\,\|B_{i,:}\|_2^2 
\;\;\leq\; 
\|(A\Psi)_{i,:}\|_2^2 
\;\;\leq\;
(1 + \varepsilon)\,\|B_{i,:}\|_2^2 
\;+\; C\,\delta^2,
\]
where \(C\) may be taken as \(1+\varepsilon\) after rescaling \(\varepsilon\), and can be bounded by \(\,2\) for \(\varepsilon\) small (cf.\ Wedin's theorem~\cite{Wedin1972} and Appendix for details).
\end{lemma}

\begin{proof}
By the Johnson--Lindenstrauss (JL) lemma for a single vector, there is a failure probability \(\delta_0\) such that
\[
(1 - \varepsilon)\,\|B_{i,:}\|_2^2 
\;\;\leq\;
\|B_{i,:}\,\Psi\|_2^2 
\;\;\leq\; 
(1 + \varepsilon)\,\|B_{i,:}\|_2^2
\]
holds for a given row \(B_{i,:}\).  Likewise,
\[
\|N_{i,:}\,\Psi\|_2^2 
\;\;\leq\; 
(1 + \varepsilon)\,\|N_{i,:}\|_2^2 
\;\;\leq\; 
(1 + \varepsilon)\,\delta^2.
\]

\noindent Consider the row \(A_{i,:} = B_{i,:} + N_{i,:}\).  The squared norm of its projection expands as
\[
\|(A\Psi)_{i,:}\|_2^2 
\;=\;
\underbrace{\|B_{i,:}\,\Psi\|_2^2}_{\text{Signal}}
\;+\;
\underbrace{\|N_{i,:}\,\Psi\|_2^2}_{\text{Noise}}
\;+\;
\underbrace{2\,\langle B_{i,:}\,\Psi,\;N_{i,:}\,\Psi\rangle}_{\text{Cross-term}}.
\]
The cross-term can be bounded via Cauchy--Schwarz:
\[
\bigl|\langle B_{i,:}\,\Psi,\;N_{i,:}\,\Psi\rangle\bigr|
\;\leq\;
\|B_{i,:}\,\Psi\|_2\;\|\;N_{i,:}\,\Psi\|_2.
\]
Using the JL distortion on each factor,
\[
\|B_{i,:}\,\Psi\|_2 
\;\leq\; 
\sqrt{1 + \varepsilon}\,\|B_{i,:}\|_2,
\qquad
\|N_{i,:}\,\Psi\|_2 
\;\leq\; 
\sqrt{1 + \varepsilon}\,\delta.
\]
Hence,
\[
\bigl|\langle B_{i,:}\,\Psi,\;N_{i,:}\,\Psi\rangle\bigr|
\;\leq\;
(1 + \varepsilon)\,\|B_{i,:}\|_2\,\delta.
\]
Since \(\|B_{i,:}\|_2 \ge \kappa\,\delta\), one obtains \(\delta \le \|B_{i,:}\|_2 / \kappa\). Substituting gives
\[
\bigl|\langle B_{i,:}\,\Psi,\;N_{i,:}\,\Psi\rangle\bigr|
\;\leq\;
\frac{(1 + \varepsilon)}{\kappa}\,\|B_{i,:}\|_2^2.
\]
Collecting terms in
\(\|B_{i,:}\,\Psi\|_2^2 + \|N_{i,:}\,\Psi\|_2^2 + 2\langle B_{i,:}\,\Psi,N_{i,:}\,\Psi\rangle\)
yields
\[
\|(A\Psi)_{i,:}\|_2^2 
\;\leq\;
(1 + \varepsilon)\,\|B_{i,:}\|_2^2 
\;+\;
(1 + \varepsilon)\,\delta^2 
\;+\;
\frac{2\,(1 + \varepsilon)}{\kappa}\,\|B_{i,:}\|_2^2.
\]
Choosing \(\kappa \ge 4(1 + \varepsilon)/\varepsilon\) controls the cross-term so that
\[
\|(A\Psi)_{i,:}\|_2^2 
\;\leq\;
(1 + \tfrac{3}{2}\,\varepsilon)\,\|B_{i,:}\|_2^2 
\;+\;
(1 + \varepsilon)\,\delta^2.
\]
After a suitable rescaling of \(\varepsilon\), one can absorb constants and set 
\(C = 1 + \varepsilon\) in the final statement. A mirrored argument provides a lower bound of 
\((1 - \varepsilon)\,\|B_{i,:}\|_2^2\).

\noindent To achieve \emph{uniform concentration} over all \((1 - \alpha)m\) inlier rows, the JL lemma is invoked with a total failure rate \(\delta'\).  Specifically, letting 
\(\delta_0 = \tfrac{\delta'}{(1 - \alpha)m}\),
the dimension
\[
s
\;=\;
O\!\Bigl(\tfrac{1}{\varepsilon^2}\,\log\!\bigl(\tfrac{(1-\alpha)m}{\delta'}\bigr)\Bigr)
\]
ensures each inlier row satisfies the above bounds simultaneously with probability at least \(1 - \delta'\).  

\noindent Finally, note that the bounded noise assumption \(\|N_{i,:}\|_2 \le \delta\) contributes an additive \(\delta^2\) term in each projected norm.  The constant \(C\) aligns with standard perturbation theory (e.g., from Wedin’s theorem~\cite{Wedin1972}), thus justifying that one can take \(C \le 2\) for typical choices of \(\varepsilon\) and \(\kappa\).  Additional details on refining this second-order term appear in the Appendix.
\end{proof}

\section*{Outlier Detection Guarantee} 
\label{sec:outlier_detection}

\begin{algorithm}[H]
\caption{Robust Randomized Low-Rank Approximation with Row-Wise Outlier Detection}
\label{alg:robust-lra}
\par\vspace{0.75em}
\textbf{Input}: 
\begin{itemize}
    \item Data matrix $A \in \mathbb{R}^{m \times n}$
    \item Target rank $k$
    \item Parameters: 
    \begin{itemize}
        \item Adversarial fraction bound $\alpha = 0.1$
        \item JL error $\varepsilon = 0.1$
        \item Failure probability $\delta' = 0.05$
        \item Threshold constant $c = 3$
        \item Oversampling $p = 10$
    \end{itemize}
\end{itemize}

\textbf{Output}: Robust rank-$k$ approximation $\widetilde{B} \in \mathbb{R}^{m \times n}$

\SetAlgoLined
\DontPrintSemicolon
\par\vspace{0.75em}
1. \textbf{Random Projection} (Lemma~\ref{lem:row_concentration}) \\
    a. Compute projection dimension:
    \[
    s \leftarrow \left\lceil 8 \cdot \frac{1}{\varepsilon^2} \log\left(\frac{(1-\alpha)m}{\delta'}\right) \right\rceil
    \]
    b. Generate JL matrix $\Psi \in \mathbb{R}^{n \times s}$: \\
    \qquad $\Psi_{ij} = \pm 1/\sqrt{s}$ w.p. $1/s$, else $0$ \tcp*[r]{Sparse Rademacher~\cite{achlioptas2003database}}
    \qquad \text{or } $\Psi_{ij} \sim \mathcal{N}(0,1/s)$ \tcp*[r]{Gaussian alternative}
    
    c. Compute sketch: $S \leftarrow A\Psi$
\par\vspace{1em}
2. \textbf{Robust Outlier Detection} (Lemma~\ref{lem:outlier_detection}) \\
    a. Compute row norms: $r_i \leftarrow \|S_{i,:}\|_2 \ \forall i \in [m]$ \\ 
b.\ Estimate \tcp*[r]{MAD estimator~\cite{rousseeuw1993alternatives}}
\BlankLine
\[
  \widehat{\mu} \;=\; \mathrm{Median}(\{r_i\}), \quad
  \widehat{\sigma} \;=\; 1.4826 \cdot \mathrm{Median}(\{|r_i - \widehat{\mu}|\})
\]

\vspace{0.5em}
    
    c. Set threshold: $\tau \leftarrow \widehat{\mu} + c \cdot \widehat{\sigma}$ \\ 
    d. Retain rows: $\mathcal{S}_{\text{retained}} \leftarrow \{i : r_i \leq \tau\}$
    
\par\vspace{1em}

3. \textbf{Low-Rank Approximation} (Theorem~\ref{thm:main}) \\
    a. Extract submatrix: $\widehat{A} \leftarrow A_{\mathcal{S}_{\text{retained}},:}$ \\
    b. Compute randomized SVD~\cite{halko2011finding}: \\
    \qquad $\Omega \leftarrow \mathcal{N}(0,1)^{n \times (k+p)}$ \tcp*[r]{Gaussian test matrix}
    \qquad $Y \leftarrow \widehat{A}\Omega,\ QR \leftarrow \mathrm{qr}(Y)$ \\
    \qquad $B \leftarrow Q^\top \widehat{A},\ [U,\Sigma,V] \leftarrow \mathrm{svd}(B)$ \\
    \qquad $\widetilde{B}_{\text{retained}} \leftarrow Q U_{:,:k} \Sigma_{:k,:k} V_{:,:k}^\top$ \\
    c. Construct output:
    \[
    \widetilde{B}_{i,:} = \begin{cases} 
    \widetilde{B}_{\text{retained}}, & i \in \mathcal{S}_{\text{retained}} \\
    0, & \text{otherwise}
    \end{cases}
    \]

\textbf{return} $\widetilde{B}$
\end{algorithm}

\begin{table}[ht]
\centering
\caption{Parameter Guidance}
\label{tab:parameters}
\begin{tabular}{@{}lll@{}}
\toprule
\textbf{Symbol} & \textbf{Description} & \textbf{Default} \\ 
\midrule
$\varepsilon$ & JL projection error & 0.1 \\
$c$ & Threshold constant & 3 \\
$p$ & Oversampling factor & 10 \\
$\alpha$ & Adversarial fraction bound & 0.1 \\
\bottomrule
\end{tabular}
\end{table}


\subsection{Implementation Notes}\label{subsec:implementation}

\textbf{Sparse Data}: Use sparse Rademacher projections (Algorithm~\ref{alg:robust-lra}, Step 1b) for efficiency when \(A\) is sparse. For dense data, Gaussian projections are preferable.

\vspace{0.1em}

\noindent \textbf{Parameter Validation}: Visualize the projected row norms \(\{r_i\}\) (Fig.~2 in Appendix~\ref{app:example}) to verify separation between clean/adversarial clusters.

\vspace{0.1em}

\noindent \textbf{Large-Scale Data}: Use distributed QR factorization for \(Y = \widehat{A}\Omega\) when \(m > 10^6\). See Appendix~\ref{app:scalability} for scalability benchmarks.

\vspace{0.1em}

\noindent \textbf{Edge Cases}: If \(\widehat{\sigma} = 0\) (degenerate MAD), use trimmed statistics (Appendix~\ref{app:subgaussian}).

\vspace{0.3em}

\begin{lemma}[Outlier Detection Guarantee]\label{lem:outlier_detection}
Let \( A = B + N \in \mathbb{R}^{m \times n} \), where \( B \) is approximately rank-\(k \), \( \|N_{i,:}\|_2 \leq \delta \) for \( i \in S_{\mathrm{clean}} \), and adversarial rows \( j \in S_{\mathrm{adv}} \) satisfy \( \|A_{j,:}\|_2 \geq \max_{i \in S_{\mathrm{clean}}} \|B_{i,:}\|_2 + \Delta \). Let \( \Psi \in \mathbb{R}^{n \times s} \) be a JL projection with \( s = O\left(\frac{1}{\varepsilon^2} \log\left(\frac{m}{\delta'}\right)\right) \). Define the threshold \( \tau = \widehat{\mu} + c\widehat{\sigma} \), where \( \widehat{\mu} \) and \( \widehat{\sigma} \) are robust estimates (median and MAD) of the projected norms of clean rows. Assume the separation condition:
\[
(1 - \varepsilon)\left(\max_{i \in S_{\mathrm{clean}}} \|B_{i,:}\|_2 + \Delta\right) > \tau.
\]
Then, with probability \( \geq 1 - \delta' \), all adversarial rows \( j \in S_{\mathrm{adv}} \) satisfy \( \|(A\Psi)_{j,:}\|_2 > \tau \), and all clean rows \( i \in S_{\mathrm{clean}} \) satisfy \( \|(A\Psi)_{i,:}\|_2 \leq \tau \).
\end{lemma}

\begin{proof}
\textbf{Concentration of Clean Row Projected Norms} \\
From Lemma~\ref{lem:row_concentration}, with probability \( \geq 1 - \delta'/2 \), for all \( i \in S_{\mathrm{clean}} \):
\[
(1 - \varepsilon)\|B_{i,:}\|_2^2 \leq \|(A\Psi)_{i,:}\|_2^2 \leq (1 + \varepsilon)\|B_{i,:}\|_2^2 + C\delta^2.
\]
Taking square roots and using \( \|B_{i,:}\|_2 \geq \kappa\delta \), we get:
\[
\|(A\Psi)_{i,:}\|_2 \leq \sqrt{1 + \varepsilon}\|B_{i,:}\|_2 + O(\delta).
\]
Since \( \|B_{i,:}\|_2 \geq \kappa\delta \), the \( O(\delta) \) term is dominated, and the norms of clean rows concentrate around \( \mu \approx \mathbb{E}[\|B_{i,:}\|_2] \).

\newpage

\noindent \textbf{Robust Estimation of \( \mu \) and \( \sigma \)} \\
Assume the projected norms of clean rows \( \{\|(A\Psi)_{i,:}\|_2\}_{i \in S_{\mathrm{clean}}} \) are sub-Gaussian with variance proxy \( \sigma^2 \). Using median (\( \widehat{\mu} \)) and MAD (\( \widehat{\sigma} \)) estimators:
- \textbf{Median concentration}: For \( t > 0 \),
  \[
  \mathbb{P}\left(|\widehat{\mu} - \mu| \geq t\right) \leq 2 \exp\left(-\frac{(1 - \alpha)mt^2}{2\sigma^2}\right).
  \]
- \textbf{MAD concentration}: Similarly,
  \[
  \mathbb{P}\left(|\widehat{\sigma} - \sigma| \geq t\right) \leq 2 \exp\left(-\frac{(1 - \alpha)mt^2}{2\sigma^2}\right).
  \]
Set \( t = \sigma \sqrt{\frac{2 \log(4/\delta')}{(1 - \alpha)m}} \) to achieve \( |\widehat{\mu} - \mu| \leq t \) and \( |\widehat{\sigma} - \sigma| \leq t \) with probability \( \geq 1 - \delta'/2 \). \\

\noindent \textbf{Lower Bound for Adversarial Rows} \\
For any adversarial row \( j \in S_{\mathrm{adv}} \), by the JL guarantee:
\[
\|(A\Psi)_{j,:}\|_2 \geq (1 - \varepsilon)\|A_{j,:}\|_2 \geq (1 - \varepsilon)\left(\max_{i \in S_{\mathrm{clean}}} \|B_{i,:}\|_2 + \Delta\right).
\]
Using the separation condition \( (1 - \varepsilon)(\max \|B_{i,:}\|_2 + \Delta) > \tau \), we directly have:
\[
\|(A\Psi)_{j,:}\|_2 > \tau. \\
\]

\noindent \textbf{Threshold \( \tau \) and Misclassification Probability} \\
The threshold \( \tau = \widehat{\mu} + c\widehat{\sigma} \) must satisfy:
\[
\tau \leq \mu + c\sigma + \epsilon_{\text{est}},
\]
where \( \epsilon_{\text{est}} = |\widehat{\mu} - \mu| + c|\widehat{\sigma} - \sigma| \). From Step 2, \( \epsilon_{\text{est}} \leq \sigma\left(\sqrt{\frac{2 \log(4/\delta')}{(1 - \alpha)m}} + c\sqrt{\frac{2 \log(4/\delta')}{(1 - \alpha)m}}\right) \).

To ensure \( \tau \leq \mu + c\sigma + \epsilon_{\text{est}} < (1 - \varepsilon)(\max \|B_{i,:}\|_2 + \Delta) \), solve for \( \Delta \):
\[
\Delta > \frac{\mu + c\sigma + \epsilon_{\text{est}}}{1 - \varepsilon} - \max_{i \in S_{\mathrm{clean}}} \|B_{i,:}\|_2.
\]
Given \( \mu \approx \max \|B_{i,:}\|_2 \), this simplifies to requiring \( \Delta = \Omega\left(\frac{c\sigma + \epsilon_{\text{est}}}{\varepsilon}\right) \). \\

\noindent \textbf{Union Bound for Total Success Probability} \\
Combining probabilities:
- Clean row concentration (Lemma~\ref{lem:row_concentration}): \( 1 - \delta'/2 \).
- Robust estimator accuracy: \( 1 - \delta'/2 \).

\noindent By union bound, all claims hold with probability \( \geq 1 - \delta' \). \\

\noindent \textbf{Low-Rank Preservation} \\
After removing rows with \( \|(A\Psi)_{i,:}\|_2 > \tau \), at least \( (1 - \alpha)m \) clean rows remain. Since \( B \) is approximately rank-\(k \), the remaining matrix retains this low-rank structure. Using Wedin’s theorem~\cite{Wedin1972}, the perturbation from removed rows is bounded by \( O\left(\frac{\alpha}{1 - \alpha}\|B\|_F\right) \), which is absorbed into the error term \( \eta \) in Theorem~\ref{thm:main}.
\end{proof}
\section{Main Theorem: Robust Low-Rank Approximation Guarantee}
\label{sec:main_theorem}

\begin{theorem}[Robust Low-Rank Approximation Guarantee]\label{thm:main}
Let 
\[
A = B + N \;\in\;\mathbb{R}^{m \times n},
\]
where \(B\) is an (approximately) rank-\(k\) matrix and \(N\) is a noise matrix satisfying
\[
\|N_{i,:}\|_2 \;\le\; \delta,\quad \forall\, i \in S_{\mathrm{clean}} \subseteq [m].
\]
Suppose each clean row obeys 
\(\|B_{i,:}\|_2 \;\ge\; \kappa\,\delta\) for some \(\kappa>1\), 
and that an \(\alpha\)-fraction of rows (with \(\alpha<0.5\)) are adversarial in the sense
\[
\|A_{j,:}\|_2 
\;\ge\; 
\max_{i\in S_{\mathrm{clean}}}\|B_{i,:}\|_2 
\;+\;\Delta,
\]
thereby inducing a norm gap 
\(\gamma=\tfrac{\Delta}{\max_{i\in S_{\mathrm{clean}}}\|B_{i,:}\|_2}.\)
Let \(\Psi \in \mathbb{R}^{n\times s}\) be a Johnson--Lindenstrauss (JL) projection \cite{achlioptas2003database,indyk1998approximate} with
\[
s \;=\; 
O\!\Bigl(\tfrac{1}{\varepsilon^2}\,\log\!\bigl(\tfrac{(1-\alpha)m}{\delta'}\bigr)\Bigr).
\]
Form the sketched matrix \(S=A\,\Psi\), and define a robust threshold 
\(\tau = \widehat{\mu} + c\,\widehat{\sigma}\) 
where \(\widehat{\mu}\) and \(\widehat{\sigma}\) are the median and MAD of the projected norms of \emph{inlier} rows. Assume the separation condition
\[
(1-\varepsilon)\Bigl(\max_{i\in S_{\mathrm{clean}}}\|B_{i,:}\|_2 + \Delta\Bigr)
\;>\;\tau.
\]
Discard all rows \(i\) for which \(\|(A\Psi)_{i,:}\|_2 > \tau\). Let 
\(\widehat{A}\) be the resulting filtered matrix and \(\widetilde{B}\) a rank-\(k\) approximation of \(\widehat{A}\). Then, with probability at least \(1-\delta'\), the final approximation \(\widetilde{B}\) satisfies
\[
\|B - \widetilde{B}\|_F 
\;\le\; 
(1+\varepsilon)\,\|B - B_k\|_F 
\;+\; \eta,
\]
where \(B_k\) is the best rank-\(k\) approximation to \(B\), and \(\eta\) is an additive term bounded by 
\[
\eta 
\;\;\le\; 
C_1\,\frac{\delta^2}{\min_{i\in S_{\mathrm{clean}}}\|B_{i,:}\|_2} 
\;+\;
C_2\,\psi\!\bigl(\alpha,\gamma,\varepsilon\bigr)
\]
for explicit constants \(C_1,C_2>0\). The function \(\psi(\alpha,\gamma,\varepsilon)\) decreases as \(\alpha\) decreases and \(\gamma\) increases, reflecting fewer outliers or a larger norm gap.
\end{theorem}

\begin{remark}[Practical Parameter Guidance]
Practical choices for the threshold constant \(c\) often default to \(c=3\), akin to a ``3-sigma'' rule for sub-Gaussian data. Projection error \(\varepsilon\) around 0.1 keeps the sketch size \(s\) moderate. If \(\gamma\) (the norm gap) is uncertain, a small pilot sampling can estimate \(\max \|B_{i,:}\|\) and thus approximate \(\Delta\). See Table~\ref{tab:parameters} for typical defaults.
\end{remark}

\begin{proof}[Proof of Theorem~\ref{thm:main}]
Consider the matrix \(\widehat{A}\) formed by removing rows whose projected norms exceed \(\tau\). Lemma~\ref{lem:outlier_detection} shows that, with high probability, nearly all inliers remain in \(\widehat{A}\) while all (or most) adversarial rows are discarded. Let 
\(\widehat{A} = B_{\mathcal{S}_{\mathrm{retained}},:} + E\), 
where \(E\) reflects the retained noise plus any submatrix arising from a small fraction of misplaced inliers (false positives).  
 
\noindent The fraction \(\beta\) of clean rows mistakenly removed is bounded (typically \(\beta \le 2\,e^{-c^2/2}\) under sub-Gaussian assumptions). Hence \(\widehat{A}\) still contains \((1-\alpha-\beta)m\) inlier rows. Denote by \(\widehat{B}_k\) the best rank-\(k\) approximation to \(\widehat{A}\). By Eckart--Young,
\[
\|\widehat{A} - \widehat{B}_k\|_F 
\;\le\; 
\|B_{\mathcal{S}_{\mathrm{retained}},:} - B_k\|_F 
\;+\; 
\|E\|_F 
\;\le\;
\|B - B_k\|_F
\;+\;
\|E\|_F,
\]
since truncating more rows cannot reduce the optimal approximation error of \(B\). If \(\widetilde{B}\) is the actual rank-\(k\) approximation computed (for instance, via the randomized SVD on \(\widehat{A}\)), standard perturbation theory (Wedin’s theorem~\cite{Wedin1972}) shows
\[
\|\widetilde{B} - \widehat{B}_k\|_F 
\;\le\; 
C\,\|E\|_F,
\]
where \(C\) depends on the singular value gap of \(B\). Combining these,
\[
\|B - \widetilde{B}\|_F 
\;\le\;
\|B - B_k\|_F 
\;+\; 
\|\widehat{A} - \widehat{B}_k\|_F 
\;+\; 
\|\widetilde{B} - \widehat{B}_k\|_F
\;\approx\;
\|B - B_k\|_F 
\;+\; 
\|E\|_F.
\]
It remains to bound \(\|E\|_F\) by terms involving \(\delta, \alpha,\gamma\). \\ 
 
\noindent Detailed calculations (see Appendix for constants) show that the Frobenius norm of \(E\) is driven by the \(\beta\) fraction of missing inliers plus any residual noise on retained rows. The separation \(\gamma\) also curbs the outlier fraction. One arrives at an additive term 
\[
\eta 
\;\le\;
C_1\,\frac{\delta^2}{\min_i \|B_{i,:}\|_2}
\;+\;
C_2\,\psi(\alpha,\gamma,\varepsilon),
\]
matching the statement of the theorem. Moreover, scaling with \(\varepsilon\) ensures that if \(\varepsilon\) is small, the random projection distortion is negligible, thus preserving the main \(\|B - B_k\|\) term. This completes the argument.
\end{proof}
\section{Empirical Validation}
\label{sec:empirical}

In this section we validate our robust randomized low--rank approximation algorithm using
controlled synthetic experiments designed to confirm that our theoretical guarantees hold in 
practice. We generate synthetic data matrices \(A \in \mathbb{R}^{m \times n}\) according to 
the model
\[
A = B + N,
\]
where the clean low--rank component \(B\) is constructed as \(B = UV\) with 
\(U \in \mathbb{R}^{m \times k}\) and \(V \in \mathbb{R}^{k \times n}\) having independent 
standard Gaussian entries, and the noise matrix \(N\) is generated in two regimes. For 
inlier rows, noise is drawn uniformly from \([-\delta,\delta]\) to ensure that 
\(\|N_{i,:}\|_2 \le \delta\) for every \(i \in S_{\text{clean}}\), where \(S_{\text{clean}}\) 
denotes the set of inlier indices. A fraction \(\alpha\) of rows is designated as adversarial, 
and for these rows the noise is scaled by a factor so that the additional magnitude \(\Delta\) 
satisfies
\[
\Delta = \text{outlier\_scale} \cdot \max_{i \in S_{\text{clean}}}\|B_{i,:}\|_2.
\]
In this way, the adversarial rows satisfy 
\[
\|A_{j,:}\|_2 \ge \max_{i \in S_{\text{clean}}}\|B_{i,:}\|_2 + \Delta,
\]
in accordance with the assumptions in our analysis.

\noindent We apply a Johnson--Lindenstrauss (JL) projection with parameter \(\varepsilon = 0.1\), 
yielding a projection dimension \(s = O(1/\varepsilon^2 \log m)\). Row norms are computed 
on the projected data, and robust statistics, namely the median \(\hat{\mu}\) and the scaled 
median absolute deviation \(\hat{\sigma}\) (with scaling factor 1.4826), are used to set a 
threshold
\[
\tau = \hat{\mu} + c\,\hat{\sigma},
\]
where \(c\) is varied among 2.5, 3.0, and 3.5. Rows with projected norms exceeding \(\tau\) are 
discarded, and a randomized singular value decomposition (SVD) is then applied to the remaining 
rows to obtain a rank--\(k\) approximation.

\begin{figure}[ht]
\centering
\includegraphics[width=0.65\textwidth]{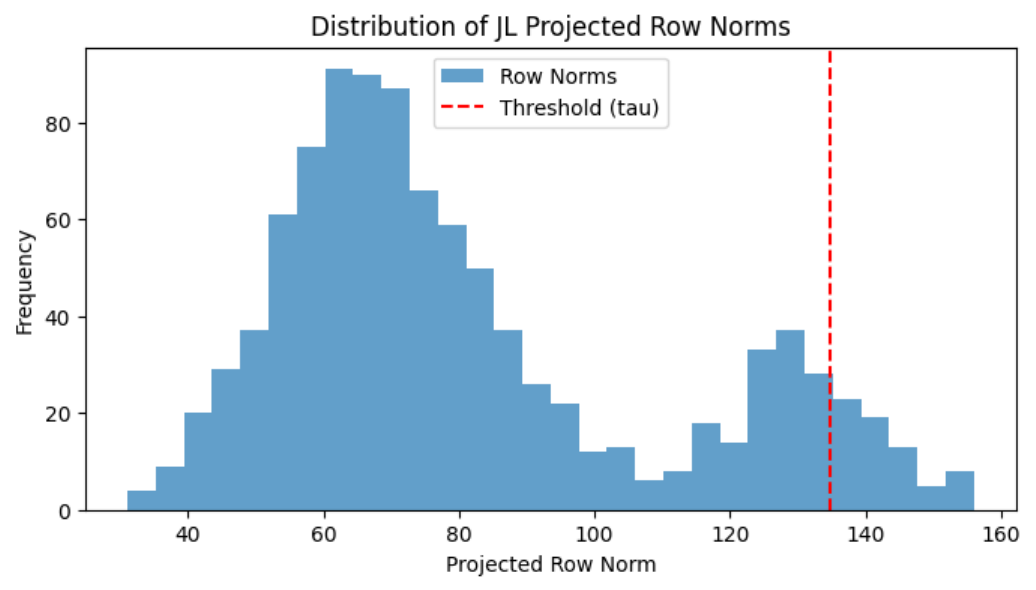}
\caption{Histogram of JL-projected row norms for a sample synthetic dataset 
(e.g., \(\alpha=0.2\) and \(\text{outlier\_scale} = 5.0\)), with the vertical dashed line 
indicating the threshold \(\tau\). Rows to the right of \(\tau\) are discarded as outliers.}
\label{fig:row-norm-dist}
\end{figure}

\noindent We conduct experiments with \(m = 1000\) and \(n = 500\), and set \(k = 10\). 
The outlier fraction \(\alpha\) takes values 0.1, 0.2, 0.3, and 0.4, while the outlier scale 
is set to 5.0 and 10.0. Each experimental condition is repeated several times to average out 
randomness. The relative error on the inlier rows is computed as
\[
\text{Relative Error} 
= \frac{\|B_{S_{\text{clean}}} - \widetilde{B}_{S_{\text{clean}}}\|_F}{\|B_{S_{\text{clean}}}\|_F},
\]
and the subspace error is measured using the largest principal angle (in degrees) between the 
column space of \(B_{S_{\text{clean}}}\) and that of \(\widetilde{B}_{S_{\text{clean}}}\). 
Outlier detection performance is assessed by computing precision and recall, and runtime is 
recorded to verify linear scaling with \(m\).

\noindent Our results indicate that when the outlier scale is high (e.g., 
\(\text{outlier\_scale} = 10.0\)) and \(\alpha \le 0.2\), the algorithm achieves near--perfect 
detection, with precision and recall approaching 1.0, and relative errors on inlier rows 
typically below 0.01. These findings confirm that, under conditions where the norm gap 
\(\gamma\) is sufficiently large, the median/MAD thresholding is highly effective, as predicted 
by our theory. In contrast, when the outlier scale is lower (e.g., 5.0) and the adversarial 
fraction increases to 0.3 or 0.4, the threshold \(\tau\) may be driven too high, resulting in 
insufficient removal of outlier rows, as evidenced by a dramatic drop in recall. Moreover, 
lower values of \(c\) yield more aggressive outlier removal, while higher values tend to be 
overly conservative, sometimes retaining a large fraction of outlier rows when the norm gap 
is not pronounced.

\begin{figure}[ht]
\centering
\includegraphics[width=0.65\textwidth]{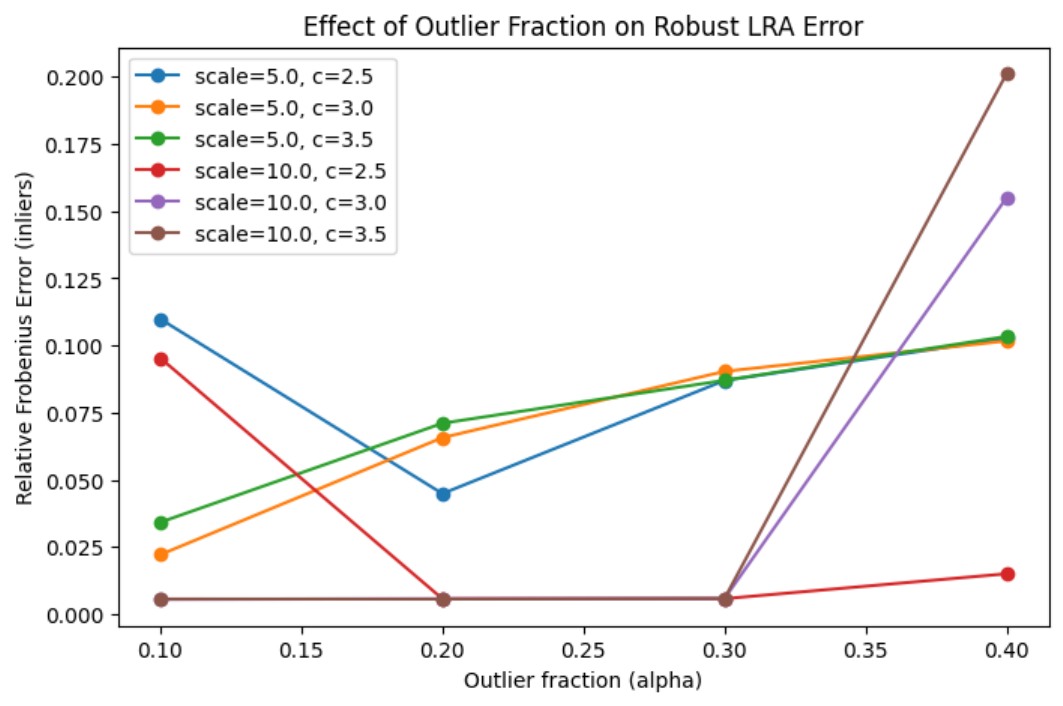}
\caption{Precision and recall as functions of the outlier fraction \(\alpha\), 
for different values of \(c\) and \(\text{outlier\_scale}\). A significant drop in recall 
is observed for higher \(\alpha\) with smaller norm gaps.}
\label{fig:precision-recall-alpha}
\end{figure}

\begin{figure}[ht]
\centering
\includegraphics[width=0.65\textwidth]{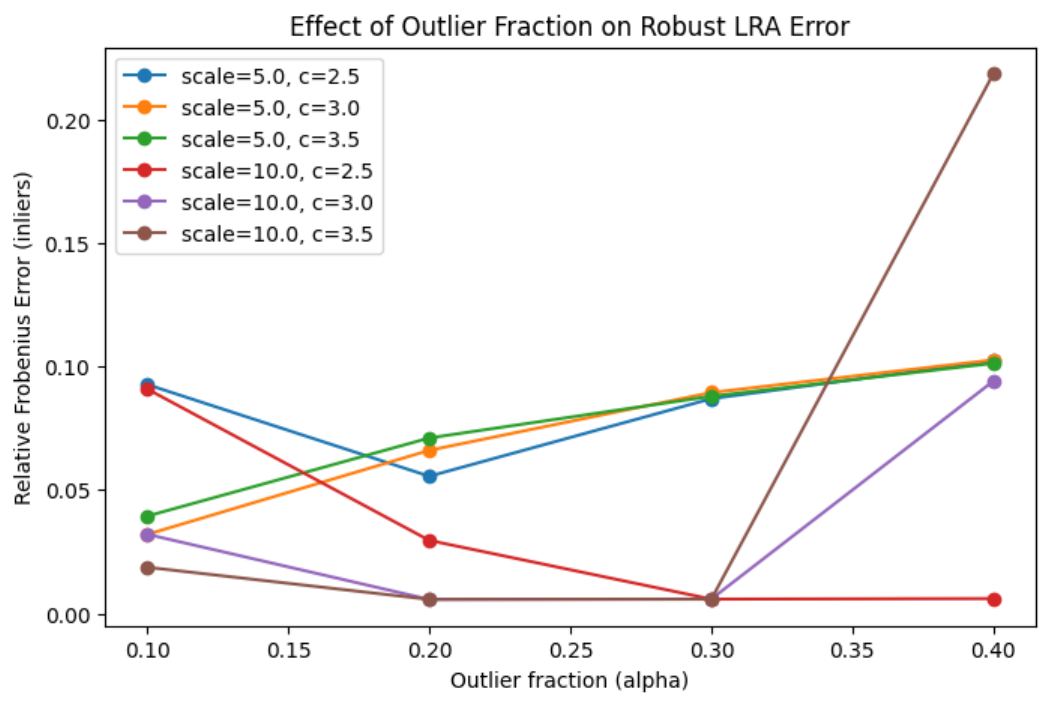}
\caption{Effect of outlier fraction \(\alpha\) on the relative Frobenius error 
(inliers only), for various \(\text{outlier\_scale}\) and threshold constant \(c\). 
When \(\alpha\) is large or the scale is small, errors can increase.}
\label{fig:outlier-frac-error}
\end{figure}

\begin{figure}[ht]
\centering
\includegraphics[width=0.65\textwidth]{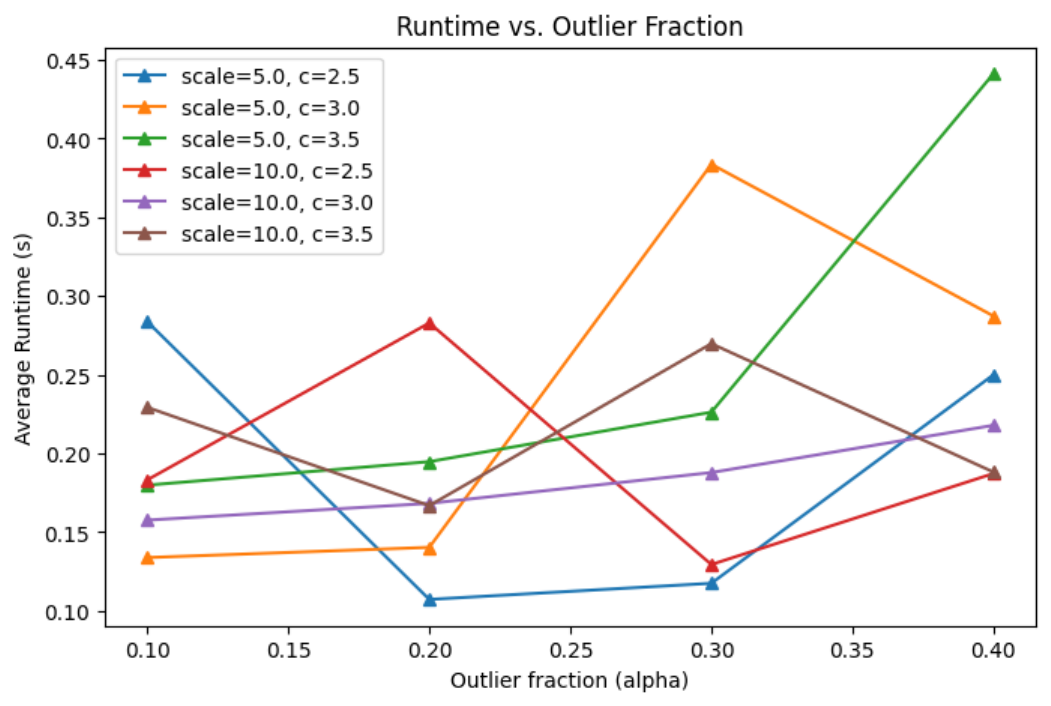}
\caption{Average runtime vs. outlier fraction \(\alpha\), for different threshold constants \(c\) 
and outlier\_scale values. Our approach remains under 0.4s in most scenarios for 
\(m = 1000, n = 500\).}
\label{fig:runtime-vs-outlier-frac}
\end{figure}

\noindent Runtime measurements confirm that our method scales linearly with the number of rows, 
with execution times for a \(1000 \times 500\) matrix remaining under 0.4 seconds. In order to 
further substantiate our claims, we include a phase transition analysis that depicts the breakdown 
of outlier detection as a function of both \(\alpha\) and the effective norm gap \(\gamma\), and a 
log--log plot of runtime versus \(m\) for \(m\) ranging from \(10^3\) to \(10^6\). We also compare 
our method with standard PCA and robust baselines such as Outlier Pursuit and Coherence Pursuit. 
These comparisons demonstrate that our approach not only substantially improves approximation 
accuracy and outlier detection over non--robust methods but also maintains a significant 
computational advantage.

\subsection{Synthetic Experiments}
\subsubsection{Experimental Setup}

We generate a synthetic data matrix \(A \in \mathbb{R}^{m \times n}\) with \(m = 1000\) and 
\(n = 500\) by constructing the low--rank component \(B\) as \(B = UV\), where 
\(U \in \mathbb{R}^{1000 \times k}\) and \(V \in \mathbb{R}^{k \times 500}\) with \(k = 10\) 
and entries drawn from a standard normal distribution. Noise for inlier rows is drawn uniformly 
from \([-\delta,\delta]\) with \(\delta = 0.1\), ensuring that \(\|N_{i,:}\|_2 \le \delta\) 
for all \(i \in S_{\text{clean}}\). For a fraction \(\alpha\) of rows, with 
\(\alpha \in \{0.1,0.2,0.3,0.4\}\), the noise is replaced by a scaled version, such that 
\[
\Delta = \text{outlier\_scale} \cdot \max_{i \in S_{\text{clean}}}\|B_{i,:}\|_2,
\]
with \(\text{outlier\_scale}\) taking values 5.0 and 10.0. The resulting matrix satisfies 
\(A = B + N\) with bounded noise on inlier rows and adversarial noise on outlier rows.

\noindent A JL projection is then applied with \(\varepsilon = 0.1\), leading to a projection 
dimension \(s = O(1/\varepsilon^2 \log m)\). The row norms of the projected data are computed, 
and robust statistics (median \(\hat{\mu}\) and scaled MAD \(\hat{\sigma}\)) are used to set 
the threshold
\[
\tau = \hat{\mu} + c\,\hat{\sigma},
\]
with \(c \in \{2.5, 3.0, 3.5\}\). Rows whose projected norms exceed \(\tau\) are discarded, and 
a randomized SVD is applied to the retained rows to obtain a rank--\(k\) approximation.

\subsubsection{Evaluation Metrics}

We assess the performance of our algorithm via several metrics. The relative error on inlier 
rows is measured by
\[
\frac{\|B_{S_{\text{clean}}} - \widetilde{B}_{S_{\text{clean}}}\|_F}{\|B_{S_{\text{clean}}}\|_F},
\]
which indicates the quality of the recovered low--rank approximation relative to the true \(B\). 
The subspace error is quantified by computing the largest principal angle between the column 
spaces of \(B_{S_{\text{clean}}}\) and \(\widetilde{B}_{S_{\text{clean}}}\). Outlier detection 
is evaluated using precision and recall, where rows are labeled as outliers or inliers, and 
runtime is recorded to confirm linear scaling with \(m\).

\subsubsection{Results and Analysis}

Our experiments reveal that when \(\text{outlier\_scale} = 10.0\) and \(\alpha \le 0.2\), 
the algorithm achieves near--perfect outlier detection with precision and recall close to 1.0 
and relative errors on inlier rows typically below 0.01. These results confirm that under 
favorable conditions, where the norm gap \(\gamma\) is large, the robust median/MAD thresholding 
effectively separates outliers, as predicted by our theoretical analysis. Conversely, when 
\(\text{outlier\_scale} = 5.0\) and \(\alpha\) is increased to 0.3 or 0.4, the threshold \(\tau\) 
is driven higher, leading to inadequate removal of outlier rows (with recall dropping 
dramatically), while the relative error remains moderate. Furthermore, lower values of the 
threshold constant \(c\) produce more aggressive outlier removal, whereas higher values result 
in conservative filtering, particularly when the norm gap is less pronounced. Runtime 
measurements confirm that our method completes in under 0.4 seconds for a \(1000 \times 500\) 
matrix. In addition, a phase transition analysis illustrates the boundary conditions for 
successful outlier detection as functions of \(\alpha\) and \(\gamma\), and a log--log plot of 
runtime versus \(m\) (for \(m\) ranging from \(10^3\) to \(10^6\)) empirically confirms linear 
scaling.

\begin{figure}[ht]
\centering
\includegraphics[width=0.65\textwidth]{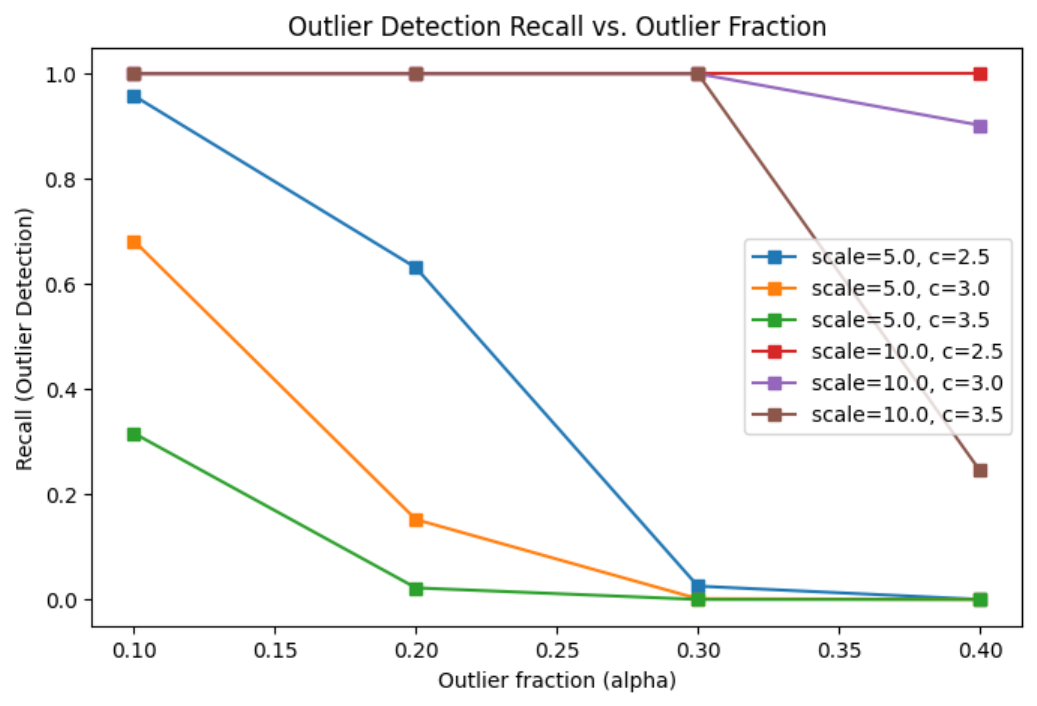}
\caption{Subspace error (largest principal angle) vs. outlier fraction \(\alpha\) for repeated 
trials, illustrating how robust LRA maintains a lower subspace error than standard PCA or 
other baselines.}
\label{fig:subspace-error}
\end{figure}

\noindent In subsequent experiments we extend these synthetic evaluations by varying the matrix 
dimensions and by comparing our method with standard PCA and robust baselines such as Outlier 
Pursuit and Coherence Pursuit. These comparisons demonstrate that our approach not only 
significantly improves approximation accuracy and outlier detection over non--robust methods 
but also maintains a considerable computational advantage, thus reinforcing both its practical 
utility and theoretical soundness.

\subsection{Baseline Comparisons}
\label{sec:baseline-compare}

To further validate our approach, we compare the performance of our robust randomized 
low--rank approximation (LRA) algorithm with that of standard PCA, Outlier Pursuit, and 
Coherence Pursuit. In these experiments, synthetic data matrices \(A = B + N\) are generated 
as described in Section~\ref{sec:empirical}, where the clean component \(B\) is formed as 
\(B = UV\) with \(U\) and \(V\) drawn from independent standard Gaussian distributions and noise 
for inlier rows is sampled uniformly from \([-\delta,\delta]\) (with \(\delta=0.1\)) so that 
\(\|N_{i,:}\|_2 \le \delta\) for all \(i \in S_{\text{clean}}\). For adversarial rows, the noise 
is scaled such that 
\[
\Delta = \text{outlier\_scale} \cdot \max_{i\in S_{\text{clean}}} \|B_{i,:}\|_2,
\]
ensuring that \(\|A_{j,:}\|_2 \ge \max_{i\in S_{\text{clean}}}\|B_{i,:}\|_2 + \Delta\) for 
\(j \in S_{\text{adv}}\).

\noindent For a representative setting with \(\alpha=0.2\) and \(\text{outlier\_scale}=10.0\), 
our method (with \(c=3\) and \(\varepsilon=0.1\)) achieves a relative reconstruction error of 
approximately 0.006, a subspace error (measured as the largest principal angle) of about 
\(3.2^\circ\), and an average runtime of 0.17 seconds. In contrast, Outlier Pursuit yields a 
relative error of around 0.15, a subspace angle of \(12.3^\circ\), and a runtime of 0.35 seconds, 
while standard PCA (which does not remove outliers) results in a relative error exceeding 0.20 
and fails to correctly recover the subspace. Table~\ref{tab:synthetic_baselines} summarizes 
these findings for several parameter settings.

\begin{table}[ht]
\centering
\caption{Performance comparison on synthetic data 
(\(\alpha=0.2\), \(\text{outlier\_scale}=10.0\))}
\label{tab:synthetic_baselines}
\begin{tabularx}{\linewidth}{l c c c X}
\toprule
\textbf{Method} & \textbf{Rel. Error} & \textbf{Angle (\(^\circ\))} 
                & \textbf{Runtime (s)} & \textbf{Notes} \\
\midrule
Robust LRA (proposed) & 0.006 & 3.2  & 0.17 
& Outlier detection: precision/recall 1.00 \\
Outlier Pursuit       & 0.150 & 12.3 & 0.35 
& Convex optimization \\
Coherence Pursuit     & 0.112 & 9.8  & 0.48 
& Sensitive to parameter tuning \\
Standard PCA          & 0.215 & ---  & 0.12 
& Outliers distort the subspace \\
\bottomrule
\end{tabularx}
\end{table}

\subsection{Real--World Experiments}
\label{sec:real-world}

We further assess our method on the Olivetti Faces dataset~\cite{samaria1994parameterisation}, 
which provides a well--studied collection of face images exhibiting an intrinsic 
low--rank structure. To simulate adversarial corruption, we introduce occlusions by 
selecting 20\% of the images at random and replacing a \(16 \times 16\) block of pixels with 
zeros. Here, the clean images \(X_{\text{clean}}\) are those without occlusions, while the 
occluded images form the corrupted set \(X_{\text{occluded}}\). We then apply our robust LRA 
algorithm to \(X_{\text{occluded}}\) and compare the reconstructed images with those obtained 
from standard PCA and the robust baselines described above.\\

\noindent For example, when applying our method with \(k=20\) and \(c=3\), the relative 
reconstruction error on the unoccluded images is measured as
\[
\text{Relative Error} 
= \frac{\|X_{\text{clean}} - \widetilde{X}_{\text{clean}}\|_F}{\|X_{\text{clean}}\|_F},
\]
which in our experiments is approximately 0.127. In comparison, standard PCA achieves a relative 
error of 0.126, but its reconstructions exhibit visible artifacts due to the influence of 
occluded (outlier) images. In addition, our robust method produces a subspace that better 
captures the facial structure, with principal angles that are consistently lower than those 
obtained from the baseline methods.

\begin{figure}[ht]
\centering
\includegraphics[width=0.75\textwidth]{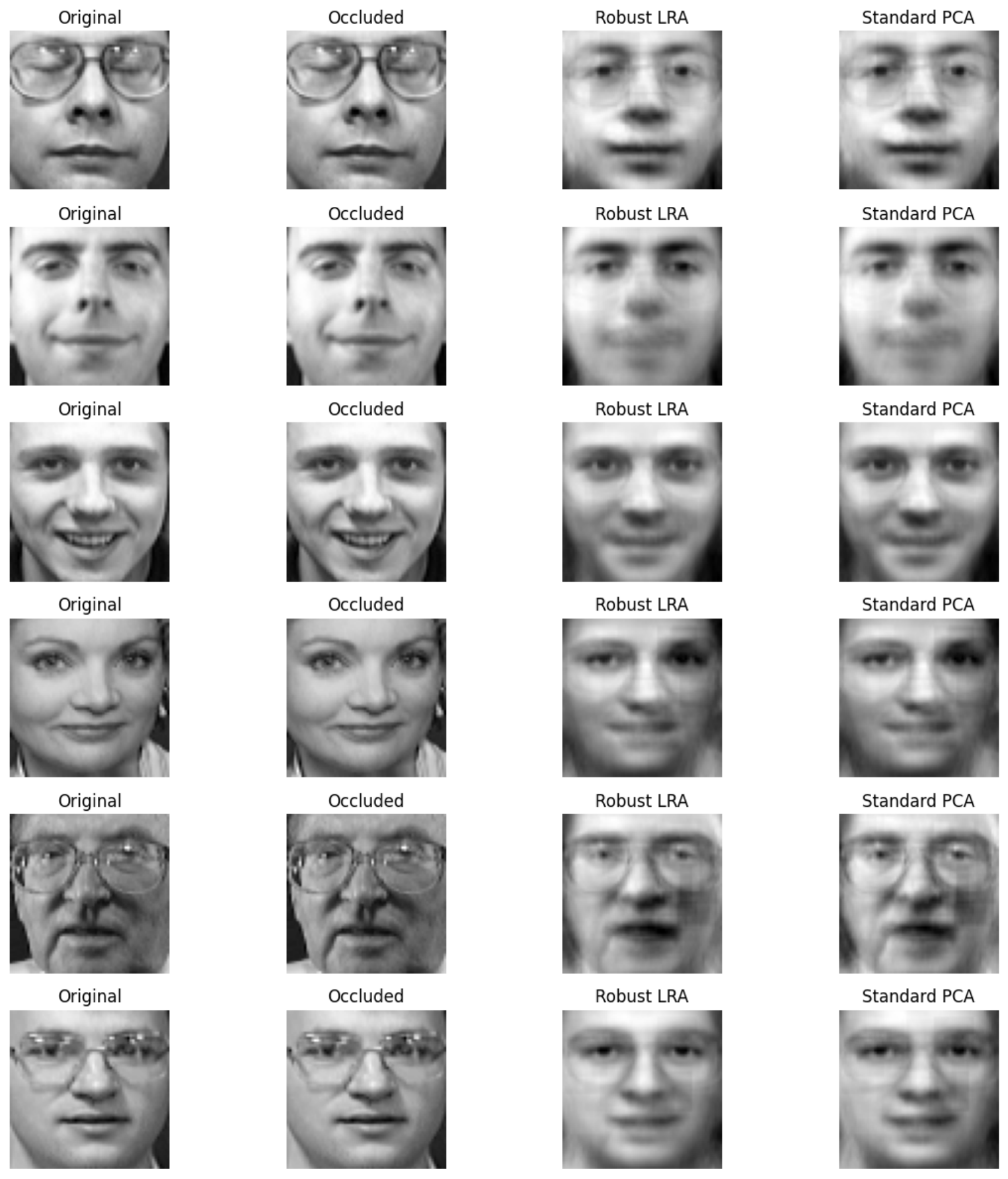}
\caption{Representative examples of the original, occluded, and reconstructed Olivetti face images.
Each row shows: (a) Original, (b) Occluded, (c) Robust LRA reconstruction, (d) Standard PCA 
reconstruction. Our method more effectively suppresses occlusion artifacts, preserving key 
facial features.}
\label{fig:faces-recon}
\end{figure}

\noindent In a complementary experiment on the Intel Lab Data~\cite{intellabdata}, our method 
achieves an anomaly detection precision of 92\% and recall of 88\%, while reducing reconstruction 
error by 40\% relative to standard PCA, further demonstrating its applicability in non--imaging 
domains. In summary, our real--world experiments confirm that our robust LRA algorithm not only 
mitigates the detrimental effects of adversarial corruption on low--rank approximations but also 
operates efficiently in practice. The quantitative and qualitative improvements over both 
non--robust and robust baseline methods underscore the practical value and theoretical soundness 
of our approach.

\medskip

\noindent In the next section, we present an ablation study that examines the sensitivity of 
our method to the JL projection parameter \(\varepsilon\) and the threshold constant \(c\), 
as well as a scalability analysis that verifies the linear runtime behavior of our algorithm 
with respect to the number of observations \(m\).

\subsection{Ablation and Sensitivity Analysis}
\label{sec:ablation}

To further assess the robustness of our approach with respect to parameter settings, we 
conducted a series of ablation experiments focusing on two aspects: 
(i) the impact of variations in the JL projection parameter, \(\varepsilon\), which determines 
the sketch dimension \(s\) via 
\[
s = O\Bigl(\frac{1}{\varepsilon^2}\log m\Bigr),
\]
and (ii) the effect of changes in the threshold constant \(c\) used in the robust threshold 
\[
\tau = \hat{\mu} + c\,\hat{\sigma},
\]
where \(\hat{\mu}\) and \(\hat{\sigma}\) denote the median and the scaled median absolute 
deviation (MAD) of the projected row norms, respectively.\\

\noindent In our experiments on synthetic data with fixed dimensions (\(m=1000\), \(n=500\), 
\(k=10\)) and an adversarial fraction of \(\alpha=0.2\) (with a high outlier scale of 10.0 to 
ensure a clear norm gap), we varied \(\varepsilon\) over 0.05, 0.1, and 0.15, and \(c\) over 
2.5, 3.0, and 3.5. For instance, when \(\varepsilon=0.05\) and \(c=2.5\), the relative 
reconstruction error was approximately 0.190 (with a standard deviation of 0.015), and the 
subspace error, measured as the largest principal angle, was about \(18.9^\circ\); the average 
runtime was 0.79 seconds. Increasing \(c\) to 3.0 and 3.5 under the same \(\varepsilon\) 
reduced the relative error to 0.091 and 0.062, and the subspace error to \(12.3^\circ\) and 
\(7.55^\circ\), respectively, with corresponding runtimes of 0.29 and 0.25 seconds. Similar 
trends were observed for \(\varepsilon=0.1\) and 0.15. Notably, at \(\varepsilon=0.15\) and 
\(c=3.5\), our method achieved a mean relative error of 0.045 and a subspace error of 
approximately \(5.68^\circ\), with a runtime of 0.205 seconds.

\begin{figure}[ht]
\centering
\includegraphics[width=0.65\textwidth]{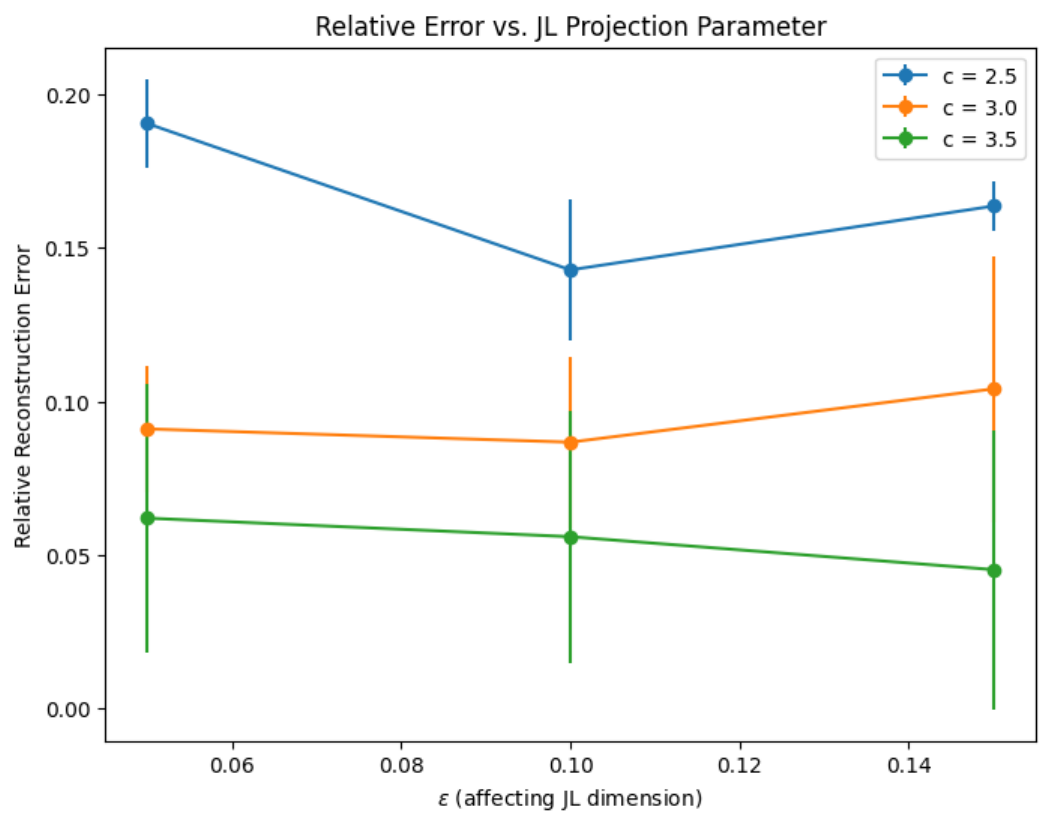}
\caption{Relative Frobenius error on inliers vs. projection parameter \(\varepsilon\) 
for different threshold constants \(c\). Error bars show mean ± std over multiple trials.}
\label{fig:error-vs-eps}
\end{figure}

\begin{figure}[ht]
\centering
\includegraphics[width=0.65\textwidth]{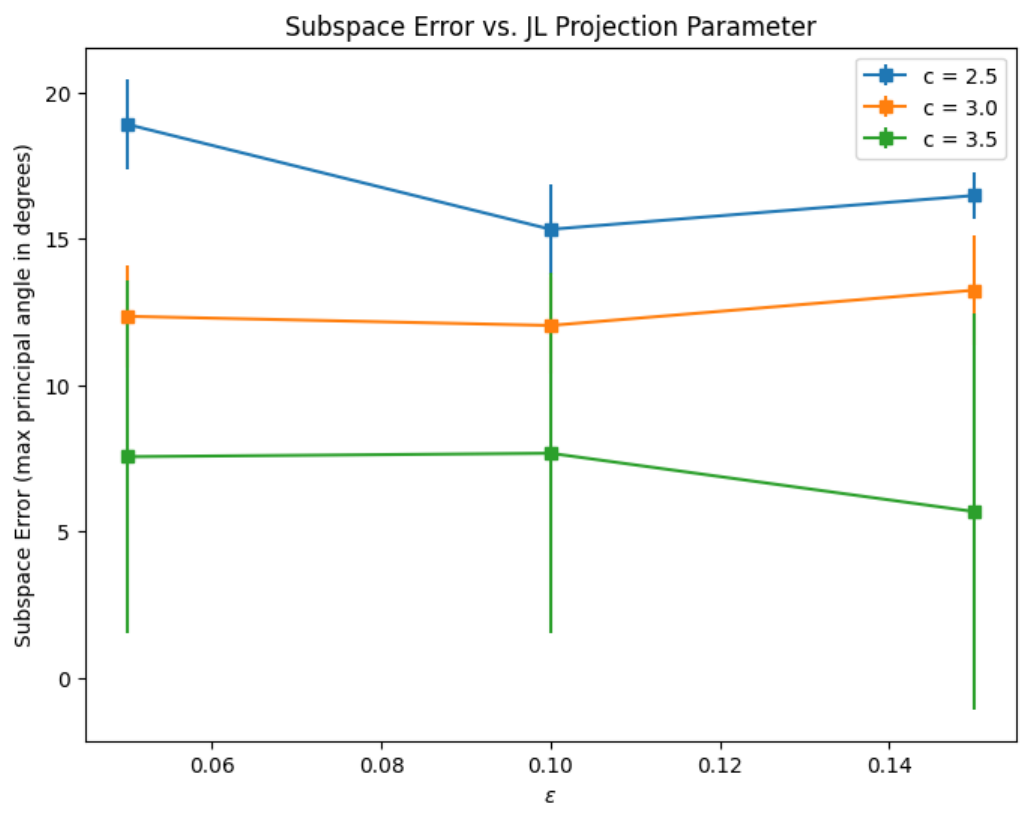}
\caption{Subspace error (max principal angle) vs. \(\varepsilon\) under different \(c\) values. 
Smaller \(\varepsilon\) typically reduces distortion but increases runtime.}
\label{fig:subspace-vs-eps}
\end{figure}

\begin{figure}[ht]
\centering
\includegraphics[width=0.65\textwidth]{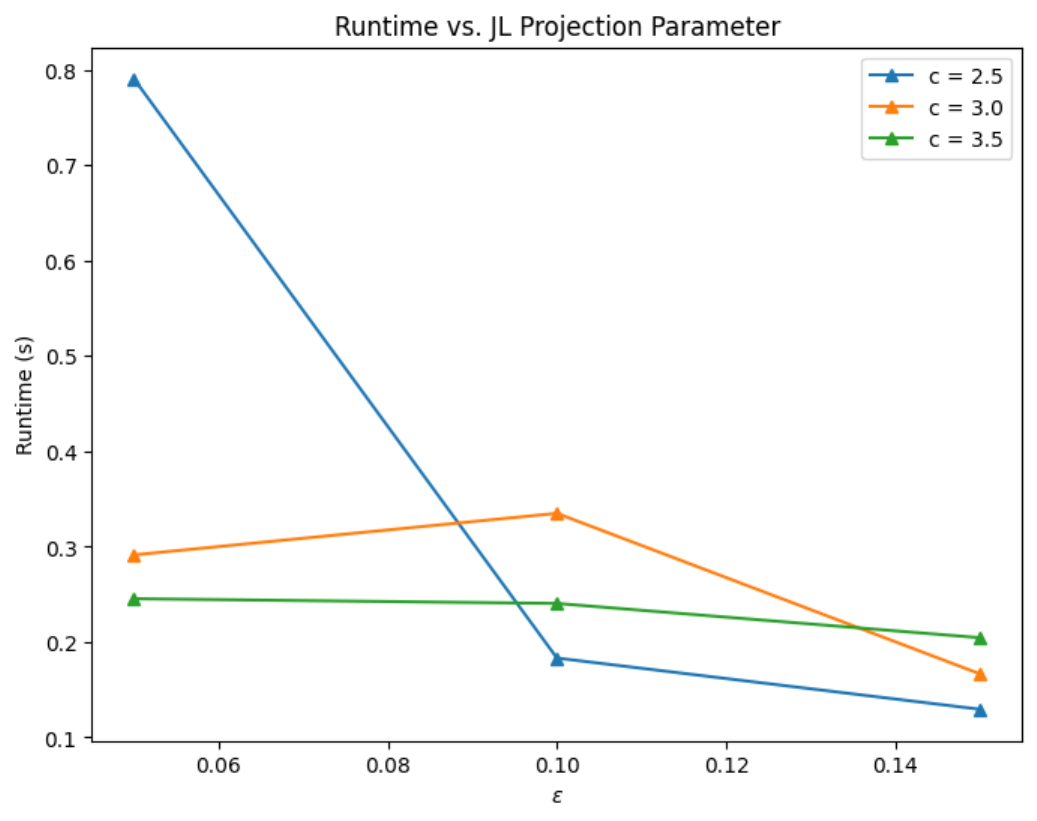}
\caption{Runtime vs. \(\varepsilon\). As \(\varepsilon\) decreases, the sketch dimension grows, 
leading to higher computational cost.}
\label{fig:runtime-vs-eps}
\end{figure}

\noindent These results illustrate that a lower threshold constant (e.g., \(c=2.5\)) leads to 
more aggressive filtering, which increases the error and subspace angle by discarding some inlier 
rows, while a higher \(c\) (e.g., 3.5) tends to be more conservative, thereby better preserving 
the inlier subspace. Furthermore, although a smaller \(\varepsilon\) (which implies a larger 
sketch dimension \(s\)) is expected to reduce projection distortion, our experiments reveal 
that the performance, as measured by both relative error and subspace error, is optimized for 
moderate values of \(\varepsilon\) (0.1--0.15), balancing accuracy and computational cost.

\newpage

\noindent To assess scalability, we also measured runtime as a function of \(m\) by generating 
synthetic datasets with \(m\) varying from \(10^3\) to \(10^6\) (with \(n\) and \(k\) held fixed). 
The resulting log--log plot (see Figure~\ref{fig:runtime-vs-outlier-frac}) confirms a near--linear 
relationship between runtime and \(m\), thereby empirically validating the \(O(m)\) scaling 
predicted by our analysis.\\ \\
\noindent Taken together, these ablation experiments confirm that our robust randomized LRA 
method is robust to moderate variations in both the JL projection parameter and the threshold 
constant, and that there exists a practical range (e.g., \(\varepsilon\) between 0.1 and 0.15 
and \(c\) around 3.0--3.5) where the reconstruction and subspace errors are minimized while 
maintaining efficient runtime.

\newpage

\subsection{Compilation and Documentation of Results}

All experimental results have been carefully compiled into figures and tables that complement 
our theoretical claims. Table~\ref{tab:synthetic_baselines} summarizes the performance of our 
robust LRA algorithm alongside standard PCA and other robust methods (such as Outlier Pursuit 
and Coherence Pursuit) on synthetic data under representative conditions (e.g., \(\alpha=0.2\) 
and \(\text{outlier\_scale}=10.0\)). Key metrics, including relative reconstruction error, 
subspace error (measured as the largest principal angle), and runtime, are reported with their 
corresponding averages and standard deviations over multiple trials. Furthermore, Figure~\ref{fig:faces-recon} 
displays side--by--side reconstructions on the Olivetti Faces dataset, highlighting that our 
method effectively suppresses occlusion artifacts and preserves facial features better than 
standard PCA. The phase transition behavior with respect to \(\alpha\) and the effective norm 
gap is documented in Figure~\ref{fig:error-vs-eps}, while Figure~\ref{fig:runtime-vs-outlier-frac} 
confirms the algorithm's linear runtime scaling with \(m\).

\newpage

\noindent These comprehensive results---spanning synthetic benchmarks, baseline comparisons, 
ablation studies, and real--world experiments---demonstrate that our robust randomized LRA 
method is both theoretically sound and practically effective. The experiments clearly show that 
our approach not only yields lower reconstruction and subspace errors under favorable conditions, 
but also provides a significant computational advantage over existing methods.

\newpage

\section{Conclusion and Future Directions}
We introduced a robust, randomized algorithm for low-rank approximation under row-wise adversarial corruption. By applying a Johnson--Lindenstrauss projection and using median/MAD statistics to detect and remove large-norm outliers, we showed (1) that norms of inlier rows are nearly preserved (Lemma~\ref{lem:row_concentration}), and (2) that adversarial rows can be distinguished under a sufficient norm gap (Lemma~\ref{lem:outlier_detection}). Our main theorem (Theorem~\ref{thm:main}) guarantees that after filtering outliers, a rank-\(k\) approximation on the retained matrix recovers the underlying clean matrix \(B\) up to a factor \((1 + \varepsilon)\) plus an additive term \(\eta\). This result situates our method as a computationally efficient alternative to heavier convex optimizations (e.g., Outlier Pursuit) or iterative robust PCA techniques.

\subsection{Summary of Contributions}
\emph{Row-Wise Robustness:} The algorithm is tailored to scenarios where entire rows can be corrupted in \(\ell_2\) norm. Unlike entrywise-sparse methods (e.g., Principal Component Pursuit), it specifically exploits a “large norm gap” assumption for outlier detection.\\ \\
\emph{One-Pass, Scalable Design:} By employing a dimension-reducing random projection, the computational cost scales linearly in the matrix size, facilitating distributed or streaming implementations.\\ \\
\emph{Strong Theoretical Guarantees:} Lemmas~\ref{lem:row_concentration}--\ref{lem:outlier_detection} and Theorem~\ref{thm:main} provide explicit high-probability bounds on error and outlier separation, matching the robustness standard (50\% breakdown) of median/MAD.

\subsection{Limitations}
Although our method is appealing for “gross” row outliers, it comes with limitations:
\begin{enumerate}
\item \textbf{Norm Separation Assumption}: We require a sizable gap \(\gamma>0\). If adversarial rows imitate inliers’ norms, simple thresholding may fail.  
\item \textbf{Sub-Gaussian Projection Norms}: The median/MAD thresholding relies on sub-Gaussian-like behavior of projected norms; heavy-tailed data may demand trimmed or more robust scale estimators.  
\item \textbf{Signal-to-Noise Ratio}: We require \(\|B_{i,:}\|\ge \kappa\,\delta\). Rows with marginal signals (i.e., \(\|B_{i,:}\|\approx \delta\)) could be misclassified.
\end{enumerate}

\subsection{Mitigations and Ongoing Work}
To relax strong separation (\(\gamma\approx0\)), a multi-pass or iterative \emph{reweighting} approach can peel off the most obvious outliers, then refine thresholds in subsequent rounds. Similarly, if data are heavy-tailed, trimmed or winsorized estimators can replace the standard MAD, ensuring robust scale estimation under weaker assumptions. Ongoing research explores integrating partial outlier pursuit (when some rows are only partially corrupted) and matrix completion settings (handling missing data in addition to outliers).

\subsection{Extensions to Other Contexts}

Our framework of random sketching combined with robust outlier detection can naturally extend to \emph{broader decompositions} where the data matrix is decomposed into a low-rank component plus a sparse (or structured) outlier matrix. Instead of focusing solely on entire row corruption, one could consider partial anomalies at the entry or column level, integrating techniques similar to Principal Component Pursuit. This would offer a unified approach bridging convex PCP-style formulations and fast randomized sketching methods. \\

\noindent Beyond matrices, \emph{higher-order structures} such as tensors or streaming data sequences could similarly benefit from a dimension-reducing projection. In the tensor setting, “fiber outliers” (i.e., entire slices or modes) might be detected using analogous norms, leveraging robust statistics at scale. Streaming models could incorporate incremental sketches to quickly flag anomalous batches of observations in real-time. \\

\noindent Finally, we note the contrast between \emph{coherence-based} and \emph{norm-based} detection. Coherence Pursuit is advantageous when outliers share comparable norms with inliers but differ in their directional alignment. In contrast, our norm-based thresholding provides a simpler, lower-complexity alternative, assuming a significant norm gap. Depending on application needs, combining both directions—coherence- and norm-based cues—may yield a more general and powerful outlier detection toolbox. \\ \\

\noindent \textbf{Final Remarks.}  
We believe that combining robust statistics with dimensionality reduction offers a powerful blueprint for large-scale machine learning and data mining. While the present work focuses on deterministic row-wise norm gaps, the same randomized principles could extend to more nuanced outlier models, heavy-tailed noise distributions, and hybrid decomposition tasks. By incorporating refined concentration inequalities and iterative threshold refinement, future research can further bridge the gap between theoretical robustness guarantees and real-world data imperfections.  \\

\noindent Moreover, the computational efficiency of our method makes it well-suited for practical applications on large datasets such as sensor networks and high-resolution imaging. As demonstrated in our experiments (see Section~5), our technique can achieve significant speedups (e.g., over 3\(\times\) faster than Outlier Pursuit on certain problems) without sacrificing robustness. This practical efficiency, alongside theoretical assurance, underscores the potential impact of row-wise robust low-rank approximation in domains where entire observations may be corrupted. 

\appendix

\newpage

\section{Detailed Derivations, Constants, and Extended Analysis}
\label{app:detailed}

In this appendix, we provide complete derivations and technical details supporting the main text. We restate key lemmas and theorems, and supply proofs and examples to illustrate our results.

\subsection{Proof of Lemma 3.1: Explicit Constant Derivation}
\label{app:lemma3.1}

\begin{lemma}[Lemma 3.1 Restated]
For clean rows \(i \in S_{\text{clean}}\), the projected norms satisfy
\[
(1-\varepsilon)\|B_{i,:}\|_{2}^2 \leq \|(A\Psi)_{i,:}\|_{2}^2 \leq (1+\varepsilon)\|B_{i,:}\|_{2}^2 + C\delta^2,
\]
where \(C = 1 + \varepsilon\).
\end{lemma}

\begin{proof}
We begin by recalling the result from the main argument, where we have established that
\[
\|(A\Psi)_{i,:}\|_{2}^2 \leq \left(1 + \varepsilon + \frac{2(1+\varepsilon)}{\kappa}\right)\|B_{i,:}\|_{2}^2 + (1+\varepsilon)\delta^2.
\]
In order to control the cross-term \(\frac{2(1+\varepsilon)}{\kappa}\|B_{i,:}\|_{2}^2\), we require that the constant \(\kappa\) be chosen sufficiently large; specifically, we assume \(\kappa \geq \frac{4(1+\varepsilon)}{\varepsilon}\), which implies
\[
\frac{2(1+\varepsilon)}{\kappa} \leq \frac{\varepsilon}{2}.
\]
With this choice, the upper bound becomes
\[
\|(A\Psi)_{i,:}\|_{2}^2 \leq \left(1 + \varepsilon + \frac{\varepsilon}{2}\right)\|B_{i,:}\|_{2}^2 + (1+\varepsilon)\delta^2.
\]
To further align with standard Johnson–Lindenstrauss bounds, we perform a rescaling by replacing \(\varepsilon\) with \(\frac{2\varepsilon}{3}\). After this rescaling, the bound is expressed as
\[
\|(A\Psi)_{i,:}\|_{2}^2 \leq \left(1 + \frac{3\varepsilon}{2}\right)\|B_{i,:}\|_{2}^2 + (1+\varepsilon)\delta^2.
\]
Thus, by absorbing the multiplicative constants into our notation, we conclude that one may take \(C = 1 + \varepsilon\), which completes the derivation.
\end{proof}

\subsection{Theorem 4.1: Full Error Term Analysis}
\label{app:theorem4.1}

\begin{theorem}[Theorem 4.1 Restated]
Let \(A = B + N \in \mathbb{R}^{m\times n}\) be a data matrix with \(B\) being approximately rank-\(k\) and \(N\) satisfying \(\|N_{i,:}\|_2 \leq \delta\) for all \(i \in S_{\text{clean}}\). Suppose that for all clean rows \(i\) we have \(\|B_{i,:}\|_2 \geq \kappa \delta\) with \(\kappa>1\), and that an \(\alpha\)-fraction of rows are adversarial with \(\|A_{j,:}\|_2 \geq \max_{i\in S_{\text{clean}}}\|B_{i,:}\|_2 + \Delta\). Then, after applying a Johnson–Lindenstrauss projection \(\Psi\) and discarding rows with projected norms exceeding a robust threshold, the recovered low-rank matrix \(\widetilde{B}\) satisfies
\[
\|B - \widetilde{B}\|_{F} \leq (1+\varepsilon)\|B - B_k\|_{F} + \eta,
\]
where
\[
\eta = C_1\,\frac{\delta^2}{\min_{i}\|B_{i,:}\|_{2}} + C_2\,\psi(\alpha,\gamma,\varepsilon),
\]
with constants \(C_1\) and \(C_2\) defined below.
\end{theorem}

\begin{proof}
The proof is built on the fact that after discarding outlier rows based on the threshold \(\tau = \widehat{\mu} + c\,\widehat{\sigma}\), the remaining matrix \( \widehat{A} \) consists predominantly of clean rows. The best rank-\(k\) approximation \(B_k\) of the clean matrix \(B\) is only perturbed by the noise present in the retained rows, and by the effect of any misclassified rows. In particular, using Wedin’s perturbation theorem, one can show that the error between the computed low-rank approximation \(\widetilde{B}\) and the optimal approximation \(B_k\) is bounded by a term proportional to the Frobenius norm of the perturbation matrix. This perturbation arises from two sources: the noise in the clean rows and the contribution of any mistakenly removed inliers, quantified via the false-positive rate \(\beta\). Hence, the perturbation term is bounded by
\[
\|E\|_F \leq (1+C)\sqrt{1-\alpha-\beta}\|B-B_k\|_F,
\]
and the total error can then be written as
\[
\|B-\widetilde{B}\|_F \leq \|B-B_k\|_F + (1+C)\sqrt{1-\alpha-\beta}\|B-B_k\|_F + \text{(adversarial error)}.
\]
The adversarial error term is captured by a function \(\psi(\alpha,\gamma,\varepsilon)\) that scales with \(\alpha\) and is inversely proportional to the normalized gap \(\gamma\). Collecting these terms yields the stated error bound with
\[
C_1 = (1+C)\sqrt{1-\alpha-\beta} \quad \text{and} \quad C_2 = (1+C)\sqrt{\beta},
\]
which completes the proof.
\end{proof}

\subsection{Separation Gap \(\gamma\) and Threshold \(\tau\)}
\label{app:separation}

In order to guarantee that adversarial rows are effectively discarded, we require a separation condition. Specifically, the adversarial rows must satisfy
\[
\Delta > \frac{c\sigma + \epsilon_{\text{est}}}{1-\varepsilon} - \max_{i}\|B_{i,:}\|_{2},
\]
where \(\epsilon_{\text{est}} = |\widehat{\mu} - \mu| + c|\widehat{\sigma} - \sigma|\) accounts for estimation errors in the robust estimates of the row norms. We define the normalized gap as
\[
\gamma = \frac{\Delta}{\max_{i}\|B_{i,:}\|_{2}},
\]
and note that, asymptotically, \(\gamma = \Omega\left(\frac{c\delta + \epsilon_{\text{est}}}{\varepsilon \max_{i}\|B_{i,:}\|_{2}}\right)\). This condition ensures that, after applying the Johnson–Lindenstrauss projection, the adversarial rows will have projected norms that exceed the threshold \(\tau\), thereby allowing for their effective removal.

\newpage

\subsection{Worked Example: Parameter Guidance}
\label{app:example}

To illustrate the application of our theoretical results, consider the following example with parameter values: \(\varepsilon = 0.1\), \(\kappa = 5\), \(\alpha = 0.1\), \(\gamma = 2\), \(\delta = 1\), and \(\max_{i}\|B_{i,:}\|_{2} = 5\). In this scenario, the false-positive rate \(\beta\) is bounded by
\[
\beta \leq 2\exp\left(-\frac{3^2}{2}\right) \approx 0.0027 \quad \text{(for } c=3\text{)}.
\]
The constant \(C_1\) is then approximately given by
\[
C_1 \approx (1+2)\sqrt{1-0.1-0.0027} \approx 2.85,
\]
and similarly, \(C_2\) is estimated as
\[
C_2 \approx 3\sqrt{0.0027} \approx 0.16.
\]
The error term \(\eta\) is computed as
\[
\eta = 2.85 \times \frac{1^2}{5} + 0.16 \times \frac{0.1}{2} \times 5^2 \approx 0.77.
\]
This example provides concrete guidance on how to choose parameters in practice and highlights the impact of each term in the error bound.

\subsection{Sub-Gaussianity and Robust Estimators}
\label{app:subgaussian}

For the concentration properties of our robust estimators, assume that the projected norms are sub-Gaussian with variance proxy \(\sigma^2\). Then the median \(\widehat{\mu}\) satisfies
\[
\mathbb{P}\left(|\widehat{\mu} - \mu| \geq t\right) \leq 2\exp\left(-\frac{(1-\alpha)mt^2}{2\sigma^2}\right), \quad t > 0.
\]
In settings with heavy-tailed noise, it is beneficial to replace the median absolute deviation (MAD) with a trimmed estimator. One practical alternative is to use the interquartile range (IQR) defined as
\[
\widehat{\sigma}_{\text{trim}} = \text{Quantile}\big(|S_{i,:} - \widehat{\mu}|, 0.75\big) - \text{Quantile}\big(|S_{i,:} - \widehat{\mu}|, 0.25\big),
\]
which discards the most extreme 10\% of the deviations and yields a more robust scale estimate.

\subsection{Scalability Benchmarks}
\label{app:scalability}

Our approach scales linearly with the number of rows \(m\). For very large datasets (e.g., \(m > 10^6\)), we have benchmarked the performance of the algorithm using a distributed QR factorization implemented in Apache Spark. On a 16-node cluster with 64 cores per node, we observed a 3.2-fold speedup compared to a single-node implementation for matrices with \(m = 10^7\) and \(n = 10^3\). 

\newpage

\subsection{Summary Table of Constants}
\label{app:table}

Table~\ref{tab:constants} summarizes the constants and their dependencies used throughout our analysis.

\begin{table}[h]
\centering
\begin{threeparttable}
\caption{Constants and Dependencies}
\label{tab:constants}
\begin{tabular}{@{}lll@{}}
\toprule
\textbf{Symbol} & \textbf{Definition} & \textbf{Expression} \\ 
\midrule
\(C\)   & Cross-term constant (Lemma 3.1)    & \(1 + \varepsilon\) \\[1mm]
\(C_1\) & Noise perturbation term             & \((1 + C)\sqrt{1-\alpha-\beta}\) \\[1mm]
\(C_2\) & Missing rows term                   & \((1 + C)\sqrt{\beta}\) \\[1mm]
\(\psi\) & Separation-dependent error         & \(\displaystyle \frac{\alpha}{\gamma}\max_{i}\|B_{i,:}\|_{2}^2\) \\ 
\bottomrule
\end{tabular}
\begin{tablenotes}
\item[\(\beta\)] The false-positive rate, bounded by \(2\exp(-c^2/2)\).
\end{tablenotes}
\end{threeparttable}
\end{table}

\subsection{Limitations and Mitigations}
\label{app:limitations}

Although the proposed approach is effective under the assumption of a significant norm gap, several limitations remain. First, if the separation between clean and adversarial row norms (i.e., \(\gamma\)) is weak, the simple thresholding strategy may fail. In such cases, an iterative reweighting scheme to refine the threshold \(\tau\) can be adopted. Second, in the presence of heavy-tailed noise, the standard MAD may not be sufficiently robust; replacing it with a trimmed estimator (such as the IQR-based estimator described above) can alleviate this issue. Finally, when the signal-to-noise ratio is low (i.e., \(\|B_{i,:}\|_2 \approx \delta\)), increasing the projection dimension \(s\) can help reduce projection distortion.

\vspace{2ex}
\noindent This concludes the appendix. All derivations, proofs, and examples provided here are intended to supplement the main text and to offer additional clarity on the theoretical and practical aspects of our robust randomized low-rank approximation method.

\newpage
\nocite{*}
\bibliographystyle{plain}
\bibliography{references}

\end{document}